\documentclass{article}

\usepackage{microtype}
\usepackage{graphicx}
\usepackage{booktabs} % for professional tables

\usepackage{hyperref}

\usepackage[accepted]{icml2025}

\usepackage{amsmath}
\usepackage{amssymb}
\usepackage{mathtools}
\usepackage{amsthm}
\usepackage{tikz}
\usepackage{tikz-3dplot}
\usepackage{physics}
\usepackage{caption}
\usepackage{subcaption}

\usepackage{enumitem}
\setenumerate{leftmargin=2em,itemsep=0.5ex,topsep=0.5ex}

\usepackage[capitalize,noabbrev]{cleveref}

\theoremstyle{plain}
\newtheorem{theorem}{Theorem}[section]
\newtheorem{proposition}[theorem]{Proposition}

\theoremstyle{definition}

\theoremstyle{remark}

\usepackage[textsize=tiny]{todonotes}

\icmltitlerunning{When can in-context learning generalize out of task distribution?}

\begin{document}

\twocolumn[
\icmltitle{When can in-context learning generalize out of task distribution?}



\icmlsetsymbol{equal}{*}

\begin{icmlauthorlist}
\icmlauthor{Chase Goddard}{princeton}
\icmlauthor{Lindsay M. Smith}{princeton}
\icmlauthor{Vudtiwat Ngampruetikorn}{equal,sydney}
\icmlauthor{David J. Schwab}{equal,cuny}
\end{icmlauthorlist}

\icmlaffiliation{princeton}{Joseph Henry Laboratories of Physics, Princeton University, Princeton, NJ, USA}
\icmlaffiliation{sydney}{School of Physics, University of Sydney, Sydney, Australia}
\icmlaffiliation{cuny}{Initiative for the Theoretical Sciences, The Graduate Center, CUNY, New York, NY, USA}

\icmlcorrespondingauthor{Chase Goddard}{cgoddard@princeton.edu}

\icmlkeywords{In-context learning, Generalization, OOD, out-of-distribution, Machine Learning, ICML}

\vskip 0.3in
]



\printAffiliationsAndNotice{\icmlEqualContribution} % otherwise use the standard text.

\begin{abstract}
  In-context learning (ICL) is a remarkable capability of pretrained transformers that allows models to generalize to unseen tasks after seeing only a few examples. We investigate empirically the conditions necessary on the pretraining distribution for ICL to emerge and generalize \emph{out-of-distribution}. Previous work has focused on the number of distinct tasks necessary in the pretraining dataset. Here, we use a different notion of task diversity to study the emergence of ICL in transformers trained on linear functions. We find that as task diversity increases, transformers undergo a transition from a specialized solution, which exhibits ICL only within the pretraining task distribution, to a solution which generalizes out of distribution to the entire task space. We also investigate the nature of the solutions learned by the transformer on both sides of the transition, and observe similar transitions in nonlinear regression problems. We construct a phase diagram to characterize how our concept of task diversity interacts with the number of pretraining tasks. In addition, we explore how factors such as the depth of the model and the dimensionality of the regression problem influence the transition.
\end{abstract}

\section{Introduction}
The ability of transformers \cite{NIPS2017_3f5ee243} to do few-shot learning from examples seen in their context is a striking phenomenon exhibited by modern large language models called \textit{in-context learning} (ICL) \cite{gpt3}. ICL has been extensively studied \cite{raventos2023pretraining,lu2024asymptotictheoryincontextlearning, garg_what_2023, chan2022, singh2023the} and enables models to solve certain new tasks without re-training. Of particular interest is how the ability for transformers to perform ICL arises from pretraining and what the limits of ICL generalization are: What conditions must be met in order for ICL to emerge and generalize \textit{outside} of the pretraining distribution?

Prior work \cite{raventos2023pretraining, lu2024asymptotictheoryincontextlearning, he2024learning} has focused on understanding how the number of pretraining tasks affects the ability of the model to generalize to unseen tasks (generated from the same distribution as the pretraining tasks). Here, we ask a related but distinct question: If a model is pretrained only on tasks from a \textit{subset} of the full task space, what conditions are necessary for it to generalize to the rest of the space? We think of this question as asking about the \textit{out-of-distribution} (OOD) performance of models trained to do in-context learning. This question prompts us to consider a more general notion of \textit{task diversity} that depends not only on task enumeration but also on how different or similar they are.
% -- a pretraining distribution with $N$ tasks that are more different in character should be considered more ``diverse'' than another distribution with $N$ tasks that are similar to each other. 
Sampling a distribution with many similar tasks has the potential to induce the model towards a more specialized ICL solution that performs well only on novel tasks \textit{within} its pretraining distribution. However, we observe that transformers trained to do ICL of linear functions undergo a transition from a specialized solution to one that generalizes over the full task space as we increase the degree of task diversity. This phenomenon of out-of-distribution \textit{task generalization} sheds new light on in-context learning behavior\footnote{Code available at \href{https://github.com/cwgoddard/OOD_ICL}{https://github.com/cwgoddard/OOD\_ICL}}.

\subsection{Related work}
Our investigation into the effects of task diversity on the emergence of ICL was motivated by the results of \citet{raventos2023pretraining} on pretraining task diversity for linear regression tasks in-context, where they find that the number of pretraining tasks impacts the emergence of ICL \textit{in-distribution}. 

Here, we investigate in-context learning of linear functions in the \textit{out-of-distribution} setting, and investigate the \textit{domain generalization} performance of transformers. Our experimental setups are intentionally similar to those in \citet{raventos2023pretraining} in order to aid comparison between in- and out-of-distribution ICL behavior. See e.g., \citet{gulrajani2020searchlostdomaingeneralization, arjovsky2021distributiongeneralizationmachinelearning,liu2023outofdistributiongeneralizationsurvey,yang2023outofdistribution} for more background on domain generalization. 

Similarly, \citet{he2024learning} look at the setting of modular arithmetic ICL tasks and analyze how the number of pretraining tasks affects generalization. \citet{lu2024asymptotictheoryincontextlearning} provide an exactly solvable model of linear attention transformers for linear regression ICL tasks and empirically show agreement of their theory with traditional transformers. Other work has analyzed how transformers implement an internal gradient-descent algorithm to learn tasks in-context, shedding light on the mechanistic properties behind ICL \cite{akyurek2023what, ahn2023transformers, pmlr-v202-von-oswald23a}. \citet{fu2024transformers} argue that transformers must be implementing second-order optimization methods instead in order to solve ICL linear regression tasks. \citet{chan2022,edelman2024the,singh2023the,nguyen2024differentiallearningkineticsgovern} investigate how properties of the data and optimization dynamics impact the emergence of ICL throughout the course of training. 

\begin{figure*}[!ht]
    \vskip 0.2in
    % \centering
    \begin{subfigure}[t]{0.54\textwidth}
        \centering
        \includegraphics{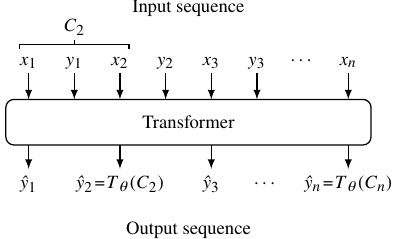} 
        \caption{}
    \end{subfigure}
    \begin{subfigure}[t]{0.44\textwidth}
    \centering
    \includegraphics{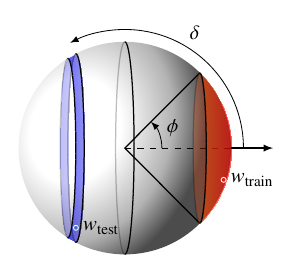} 
    % \begin{tikzpicture}
    %   \def\R{1.5cm}
    %   \def\r{1.5}
    %   \draw (0,0) circle (\R);
    %   \draw[dashed] (0,0) -- (\R,0);
    %   \draw (0,0) -- node[midway, anchor=south east] {$R$} ++(60:\R);
    %   \draw (0,0) -- ++(-60:\R);
    
    %   \useasboundingbox (current bounding box);
    
    %   \draw[fill=gray,fill opacity=0.25]
    %       (0,0) ++(60:\R)                  % start point
    %       arc (90:-90:-0.3 and {\r*sin(60)}) % arc for section, foreground
    %       arc (-60:60:\R);                 % arc for main circle
    %   \draw[densely dashed,fill=gray,fill opacity=0.25]
    %       (0,0) ++(60:\R)                  % start point
    %       arc (90:-90:0.3 and {\r*sin(60)})  % arc for section, background
    %       arc (-60:60:\R);                 % arc for main circle
    
    %   \node at (0.35cm,0.25cm) {$\phi$};
    %   \draw[line width=2pt,-stealth](0,0)-- ++(-30:\R) node[anchor=north west]{$\boldsymbol{w_\text{train}}$};
    % \end{tikzpicture}
    \caption{} \label{fig:train-dist}
    \end{subfigure}
    % \begin{subfigure}[t]{0.3\textwidth}
    % \centering
    % \newcommand\latitudes[3][]{%
    %     \draw[dashed] (#2:\r) arc (0:-180:{\r*cos(#2)} and {0.25*cos(#2)});
    %     \draw[dashed] (#3:\r) arc (0:-180:{\r*cos(#3)} and {0.25*cos(#3)}); 
    %     \draw[#1] (#2:\r) arc (0:180:{\r*cos(#2)} and {0.25*cos(#2)}) 
    %     arc (180-#2:180-#3:\r)
    %     arc (180:0:{\r*cos(#3)} and {0.25*cos(#3)}) 
    %     arc (#3:#2:\r);}
    % \begin{tikzpicture}
    %     \def\R{1.5cm}
    %     \def\r{1.5}
    %     \begin{scope}[rotate around={90:(8,0)},shift={(8,0)}] 
    %     \draw (0,0) circle (\R); 
    %     \draw (0,0) -- node[midway, anchor=south east] {$R$} ++(-15:\R);
    %     \draw (0,0) -- node[midway, anchor=south east] {} ++(-40:\R);
        
    %     \latitudes[fill=gray, fill opacity=0.25]{-15}{-40}
        
    %     \node at (0.15cm,0.25cm) {$\Delta\delta$};
    %     \draw[line width=2pt,-stealth](0,0)-- ++(205:\R) node[anchor=south west]{$\boldsymbol{w_\text{test}}$};
    %     \end{scope}
    % \end{tikzpicture}
    % \caption{} \label{fig:test-dist}
    % \end{subfigure}
    \caption{\label{fig:setup}%
    \textbf{Testing ICL generalization via task similarity.} 
    \textbf{A:} The transformer takes as input a sequence of pairs $\{x_i,y_i\}_{i=1}^n$ and is trained to predict $y_k$ from a context $C_k = \{x_1, y_1, \ldots, x_k\}$. The elements $x_i$ and $y_i$ are related linearly by a task $w$ via $y_i = w^Tx_i + \epsilon_i$. 
    \textbf{B:} The training tasks $w_\mathrm{train}$ are drawn from a hyperspherical cap with half-angle $\phi$ (with $\phi=180^\circ$ corresponding to the entire hypersphere).
    The test tasks $w_\mathrm{test}$ are drawn from a hyperspherical band of width $\Delta\delta$ starting an angle $\delta$ away from the ``pole'' of the sphere.}
\end{figure*}
\subsection{Distribution Shift \& Generalization}

 In-context learning is a powerful capability of language models, and in order to build trust in models, we should better understand how well such capabilities extend to novel tasks beyond those in the training data.\footnote{We note that another perspective is to train on as large and diverse a pretraining dataset as possible so nothing is out of distribution, but there will always be novel tasks that have not been encountered before.} In addition to generalizing to tasks that interpolate between those seen in training, we'd like models that generalize to genuinely novel tasks. Whether this is possible depends on the nature of the distribution shift, and this question is in the purview of out-of-distribution generalization but at the level of tasks instead of samples.
    %\item we build on existing scientific frameworks of icl emergence to ask whether and when such out of distribution task generalization emerges

In this work, we characterize when transformers develop in-context learning behaviors that are robust to task distribution shift. Since ICL can be thought of as generalization to novel tasks, here we ask when models succeed at \emph{domain generalization} with respect to tasks. Indeed, a common assumption in the literature on out-of-distribution generalization \cite{gulrajani2020searchlostdomaingeneralization, arjovsky2021distributiongeneralizationmachinelearning, liu2023outofdistributiongeneralizationsurvey, yang2023outofdistribution} is the so-called “covariate shift” assumption where the conditional distribution of the label given the input is held fixed, but the distribution of inputs changes at test time. In the context of tasks, we consider sampling tasks during training from a subset of the task space and ask about generalization to tasks outside of that subset. All tasks share the common property of being linear relationships, and thus we are studying domain generalization of tasks.

\subsection{Contributions}
Our core contributions are as follows: 
\begin{itemize}
    \item We train transformers to exhibit ICL of linear functions with weight vectors drawn from a subset of the unit hypersphere. As the size of this subset increases, we observe a transition from specialized models, which perform well only on the training portion of the hypersphere, to models that generalize out-of-task-distribution to the entire hypersphere. 
    %\item We show empirically that label noise shifts the location of this specialization-generalization transition; the transformer must be trained on tasks from a larger subset of the hypersphere in order to generalize to the whole sphere. 
    \item We investigate the nature of the solutions found by our transformers, and find that specialized solutions outperform optimal Bayesian solutions to the regression problem on small numbers of examples. In contrast, transformers that generalize to the entire hypersphere exhibit performance similar to optimal solutions. 
    %\item We investigate the ICL performance of models as test tasks move off the unit hypersphere. We observe good performance for tasks within the hypersphere, but our models fail to generalize to tasks far outside the hypersphere.
    \item We examine how two notions of task diversity (number of tasks and task similarity\footnote{For linear problems, it is natural to choose the inner product between tasks $w_1^Tw_2$ as a similarity measure.}) interact, and construct a phase diagram that reveals three distinct regimes of ICL generalization. 
    %\item We investigate the robustness of such specialization-generalization transitions to changes in model architecture and regression problem dimension. 
    \item We show that specialization-generalization transitions also occur in nonlinear regression problems, suggesting that the phenomenon may be a general feature of ICL in transformers. 
\end{itemize}

\section{Training setup and task distribution geometry}

\paragraph{ICL of linear functions:} We investigate the ability of transformers to perform in-context learning of linear functions when pretraining tasks are drawn from distributions with varying levels of \textit{task diversity}.
Specifically, each \emph{task} is a linear map in $d$ dimensions, $w\in \mathbb{R}^d$, and we control task diversity by sampling tasks from hyperspherical caps of varying half-angles. 
% 
% i.e. from hyperspherical caps of varying half-angles. 
% We define a \textit{task} to be a vector $w\in \mathbb{R}^d$, and 
The transformer takes as input a sequence of up to $n$ pairs $\{x_1,y_1,\ldots,x_n\}$, where $y_i = w^Tx_i + \epsilon_i$, with $x_i \sim \mathcal{N}(0, I_d)$ and $\epsilon_i \sim \mathcal{N}(0,\sigma^2)$. 

\paragraph{Pretraining task distribution:} We define a family of \textit{task distributions} parameterized by $\phi \in [0,\pi]$ (See Fig \ref{fig:setup}B). We take $S^{d-1}(\phi)$ to be a section of the surface of the hypersphere in $d$ dimensions, i.e. $S^{d-1}(\phi) = \{w \in S^{d-1} \mid \text{angle}(w,v) \leq \phi\}$, with $v\in \mathbb{R}^d$ a fixed vector. We then define the task distribution as a uniform distribution on this spherical cap, i.e., $p_\phi(w) \equiv \text{Unif}(S^{d-1}(\phi))$. For details on how to effectively sample from such a distribution, see section \ref{sect:sampling} in the appendix. 

\paragraph{Pretraining:} During pretraining, the transformer $T_\theta$ is optimized to minimize the mean squared error (MSE) between a \textit{context} of data $C_k \equiv \{x_1, y_1, \ldots, x_k\}$ and the target $y_k$. During pretraining, the tasks $w$ are drawn i.i.d.\ for each context from $p_\phi(w)$. We use AdamW \cite{adamw} to optimize the MSE,
\begin{equation}
    L_\text{train}(\theta) = \mathbb{E}_{w\sim p_\phi}\qty[\frac{1}{n}\sum\nolimits_{k=1}^n \qty(T_\theta(C_k) - y_k)^2].
\end{equation}

%%%%%%%%%
\begin{figure}[!t]
    \vskip 0.2in
    \centering
    \includegraphics{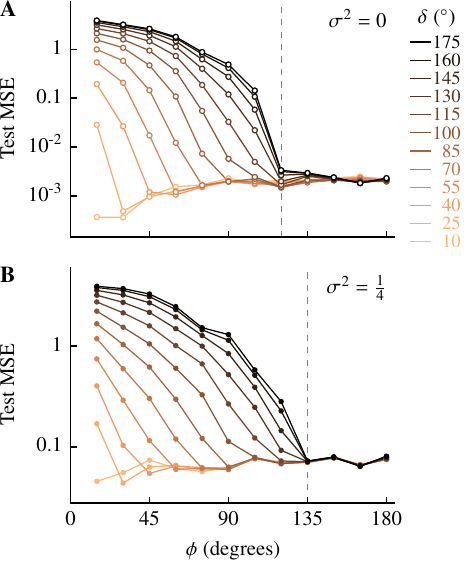}
    \caption{\textbf{%
        Task distribution diversity induces a transition from specialized to general-purpose ICL.} 
        \textbf{A:} Test error in $\Delta\delta = 5^\circ$ bands (see Fig \ref{fig:setup}) for transformers pretrained to do in-context learning of linear functions with pretraining task distributions $p_\phi(w)$. For distributions with $\phi \lesssim 120^\circ$, the transformer learns a specialized solution that performs well on unseen tasks drawn from the $p_\phi(w)$, but fails for tasks outside this distribution. However, for pretraining distributions with $\phi \gtrsim 120^\circ$, the transformer learns a solution that performs well for all test angles $\delta$. Here, the label noise is $\sigma^2 = 0$.
        \textbf{B:} With $\sigma^2 = 0.25$, we still observe a transition from a specialized to a generic solution, but the transition point has moved to $\phi \approx 135^\circ$. The vertical axis measures the excess test error above the noise floor set by $\sigma^2$.
    }
    \label{fig:transition}
\end{figure}
%%%%%%%%%

\paragraph{Test task distribution:} We evaluate the performance of the transformer over a family of task distributions parameterized by $\delta,\Delta\delta \in [0,\pi]$ (See Fig \ref{fig:setup}B). We define the hyperspherical band starting at angle $\delta$ with width $\Delta\delta$ to be the set $B^{d-1}(\delta, \Delta\delta) = \{w \in S^{d-1} \mid \delta \leq \text{angle}(w,v) \leq \delta + \Delta\delta\}$, with $v$ some fixed vector. The test task distribution is then uniform over this set: $p_{\delta, \Delta\delta}(w) = \text{Unif}(B^{d-1}(\delta, \Delta\delta))$. For more on how to sample from these bands, see section \ref{sect:sample-min} in the appendix. 

\paragraph{Evaluation:} We evaluate models by computing the MSE between the full context $C_n$ and the final target $y_n$. During test time, we draw $w$ i.i.d. for each context from $p_{\delta,\Delta\delta}$:
\begin{equation}
    L_\text{test}(\theta) = \mathbb{E}_{w\sim p_{\delta,\Delta\delta}}\qty[ \qty(T_\theta(C_n) - y_n)^2]
\end{equation} 

\section{Experimental Results}
Unless stated otherwise, we study $d=10$ dimensional regression with $n=50$ examples in each context. We use a GPT-2 style transformer \cite{radford2019language} with learned positional embeddings, a hidden dimension of $d_h = 128$, 10 layers (except in Fig \ref{fig:depth}), and 8 attention heads. We use a learned linear embedding to map $x_i$ and $y_i$ to the hidden dimension $d_h=128$. The target values $y_i$ are padded with $d-1$ zeroes. For further training details, see Appendix section \ref{sect:more-expt}.

During pretraining, we train 12 models over pretraining distributions $p_\phi(w)$ for $\phi \in [15^\circ, 180^\circ]$ in $15^\circ$ increments. We observe that repeated runs, with different initializations and trained on data generated from different sampled tasks $w\sim p_\phi$, yield consistent results\footnote{During training, the loss generically ``plateaus,'' staying at a constant value for many steps before beginning to decrease. This phenomenon has been observed before, as in \cite{fu_breaking_2024}.} (see Fig \ref{fig:run2}). For small $\phi$, we also see the model go through two stages of specialization over the course of training (see Appendix \ref{sect:spec2}), the first of which is before this plateau.

\subsection{Modes of out-of-distribution generalization}
Before we investigate experimental results, it is instructive to discuss the ways in which models may (or may not) generalize out-of-task-distribution:
\begin{enumerate}
    \item Models may fail to meaningfully generalize out-of-task-distribution. (For example, when prompted to solve a task $w_\text{test}$ not in the support of the pretraining distribution $S^{d-1}(\phi)$, models may simply pick the task $w_\text{close} \in S^{d-1}(\phi)$ that is closest to $w_\text{test}$, but fail to generalize beyond this level.) 
    \item Models may generalize out-of-task-distribution, but only achieve maximum performance when pretrained on the \textit{entire} task space ($\phi = 180^\circ$) . 
    \item Models may generalize out-of-task-distribution in a \textit{sharp}    way with increasing $\phi$: there may be some $\phi_c <180^\circ$ such that models achieve maximum performance for all $\phi \geq \phi_c$.
\end{enumerate}
We view the last option as perhaps the most striking: it implies that models can achieve optimal performance with incomplete data. In the following, we show that transformers can generalize out-of-task-distribution in this fashion. 

\subsection{Specialization-generalization transition}

We show the results from evaluating these models on the test task distributions $p_{\delta, \Delta\delta}(w)$ in Fig \ref{fig:transition}. We pick $\Delta\delta = 5^\circ$ and examine a range of $\delta$ (see legend). For models with a pretraining task distribution with $\phi \lesssim 120^\circ$, we observe good test performance only within the portion of the hypersphere covered by the pretraining distribution, and performance degrades outside of this range. However, for models trained over $p_\phi(w)$ with $\phi \gtrsim 120^\circ$, we observe low test error for all $\delta$. In fact, the model performs similarly for all test angles $\delta$, suggesting that transformers learn a general-purpose solution in this regime. This occurs despite the fact that these models were trained using only data restricted to a \textit{subset} of the full task space. Notice also that even before the transition, models trained on a cap with $\phi\geq 45^\circ$ exhibit nontrivial out-of-distribution task generalization (Fig \ref{fig:bayes}). 

%\begin{figure}[!ht]
%    \centering
%    \includegraphics[width=1.0\linewidth]{task-gen-noise.png}
%    \caption{\textbf{Noise in the regression problem moves the transition.} With $\sigma^2 = 0.25$, we still see a transition from a specialized solution to a generic solution, but the transition has moved to $\phi = 120^\circ$.}
%    \label{fig:task-gen-noise}
%\end{figure}

In Fig \ref{fig:transition}B, we see that the transition is sensitive to the level of label noise $\sigma^2$. For noisy regression with $\sigma^2 = \frac{1}{4}$, we see that the transition now occurs around $\phi\sim135^\circ$, but that the qualitative behavior of the transition is unchanged.

%%%%%%%%%
\begin{figure}
    \vskip 0.2in
    \centering
    \includegraphics{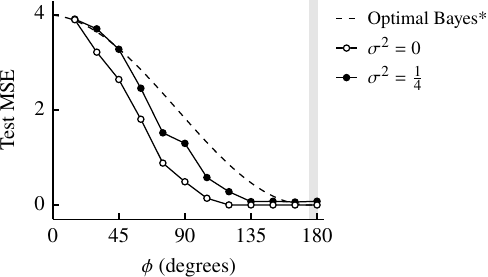}
    \caption{\textbf{%
        Pretrained transformers outperform Bayes-optimal solutions in out-of-task-distribution generalization.} For $\delta=175^\circ$ (shaded grey region), we plot the excess test loss for models with varying pretraining distributions. The dashed line shows the test error for the optimal in-task-distribution Bayesian solution (see section \ref{sec:optimal-bayes}).
    }
    \label{fig:bayes}
\end{figure}
%%%%%%%%%

\paragraph{Optimal Bayes \& OOD generalization\label{sec:optimal-bayes}}
How well do transformers generalize out-of-task-distribution? Here, we compare their out-of-task-distribution performance to that of the optimal Bayesian estimator for a given $p_\phi(w)$. Following \cite{raventos2023pretraining}, we derive an expression for the optimal estimator of $y_n$ under the prior $p_\phi(w)$. Observe that in order to minimize $L_\text{test}(\theta)$, the optimal estimator is given by the posterior mean of $y_n$ conditioned on the full context $C_n$, 
\begin{align*}
    \mathbb{E}[y_n|C_n] &= \int\dd{w}\dd{y_n}y_np(w,y_n | C_n) \\
    &=\int\dd{w}\dd{y_n}y_np(y_n|x_n,w)p(w|C_{n-1},y_{n-1}) \\
    &= \int\dd{w}w^Tx_np(w|C_{n-1},y_{n-1})\\
    &\equiv \hat w^Tx_n
\end{align*}
with
\begin{equation}
    \hat w = \frac{\int\dd{w}p(w)w\prod_{k=1}^{n-1}p(y_k|x_k,w)}{\int\dd{w}p(w)\prod_{k=1}^{n-1}p(y_k|x_k,w)}\label{eqn:opt-w}
\end{equation}
We now consider the choice $p(w) = p_\phi(w)$. Although the integrals in Eqn \ref{eqn:opt-w} are intractable in this case, it is clear that because $p_\phi(w)$ lacks support outside of $S^{d-1}(\phi),$ $\hat w$ must always be contained in $S^{d-1}(\phi)$. The optimal Bayesian estimator ($w_\text{OB})$ for this problem therefore fails to meaningfully generalize out-of-task-distribution, as the best one can hope for is to assign $\hat w$ to that vector $w_\text{OB} \in S^{d-1}(\phi)$ that is closest to the target task.

In Fig \ref{fig:bayes}, the test loss for $w_\text{OB}$ is given by the dashed line. Notice that the ability of pretrained transformers to outperform this optimal estimator out-of-task-distribution is therefore a direct consequence of the failure of these transformers to fit the optimal Bayesian solution. It remains an interesting avenue for further work to ask what form the prior $p(w)$ should take in order to give rise to the out-of-task-distribution performance we see here. 

%%%%%%%%%
\begin{figure}
    \vskip 0.2in
    \centering
    \includegraphics{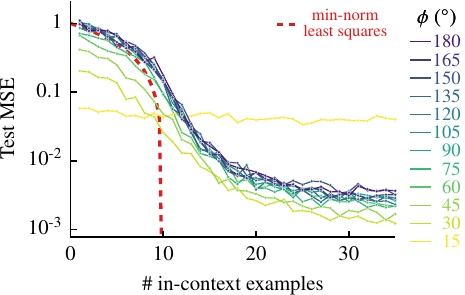}
    \caption{\textbf{Specialized ICL outperforms OLS for small context length.} We evaluate the models in-task-distribution for varying context lengths, and plot the performance of the transformer (solid) and ordinary least squares (dashed) for the same data. For low context length, the specialized solution learned by models with $\phi \lesssim 90^\circ$ outperforms OLS. For $\phi = 15^\circ$, the specialized solution is worse than OLS for large context length.
    }
    \label{fig:contextlength}
\end{figure}
%%%%%%%%%

\paragraph{Comparison to ordinary least squares}
We next investigate the solutions learned by the transformer on both sides of the transition. In Fig \ref{fig:contextlength}, we compare 
the performance of the transformer and ordinary least squares (OLS), solid and dashed curves, respectively. For short contexts, the specialized solution which is learned for $\phi \lesssim 120^\circ$ outperforms OLS within the task distribution. For $\phi \gtrsim 120^\circ$, the performance of the transformer is similar to OLS, except that the models' test error is not identically zero after $d=10$ examples, unlike the least-squares solution. This sheds light on the nature of the specialization occurring: by fitting a strong prior to the pretraining data, models with low $\phi$ sacrifice out-of-distribution performance, but this bias enables them to outperform the general-purpose solution (OLS) for low context length. See also Appendix Figs \ref{fig:noisy-context} \& \ref{fig:ood-context} for models' performance for varying context length in other settings.

% In Appendix Figs \ref{fig:noisy-context} \& \ref{fig:ood-context}, we explore the models' performance for varying context length in other settings.

%%%%%%%%%
\begin{figure}
    \vskip 0.2in
    \centering
    \includegraphics{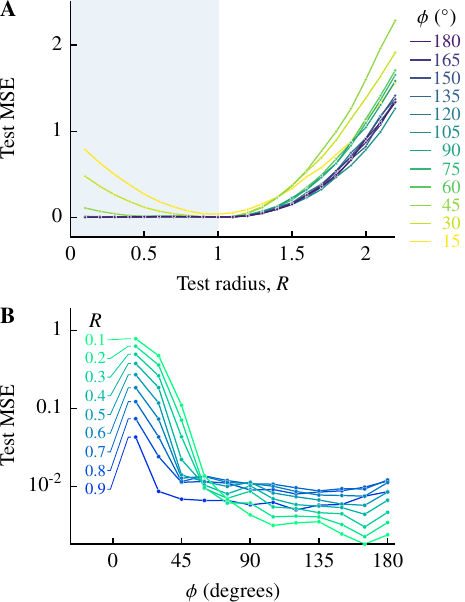}
    \caption{\textbf{Transformers trained to do ICL on the sphere generalize beyond it.} \textbf{A:} The test error for tasks drawn uniformly from subsets of a hypersphere of radius $R$, when a model is pretrained on tasks taken only from subsets of the \textit{unit} hypersphere. When $\phi \gtrsim 45^\circ$, the model generalizes to tasks with $R<1$ (shaded), despite being pretrained with $R=1$. \textbf{B:} Increasing task diversity drives generalization beyond the sphere: With sufficient task diversity ($\phi \gtrsim 45^\circ$), transformers generalize not only to OOD tasks \textit{on} the sphere (Fig \ref{fig:transition}), but also to OOD tasks \textit{within} it. 
    }
    \label{fig:radius}
\end{figure}
%%%%%%%%%

%\begin{figure}[!ht]
%    \centering
%    \includegraphics[width=1.0\linewidth]{context.png}
%    \caption{\textbf{The specialized solution outperforms OLS for tasks inside the training task distribution.} We evaluate the Transformers in-distribution (solid) vs ordinary least squares (dashed) for varying context lengths. For low context length, the specialized solution learned by models with $\phi < 90^\circ$ outperforms OLS. For $\phi = 15^\circ$, the specialized solution is worse than OLS for large context length. \textcolor{red}{Are the dotted lines visible? How do I make this less busy? Need to think tomorrow}}
%    \label{fig:task-gen-context}
%\end{figure}

\begin{figure*}[!ht]
    \vskip 0.2in
    \centering
    \begin{subfigure}[t]{0.32\textwidth}
        \centering
        \includegraphics[width=1.0\linewidth]{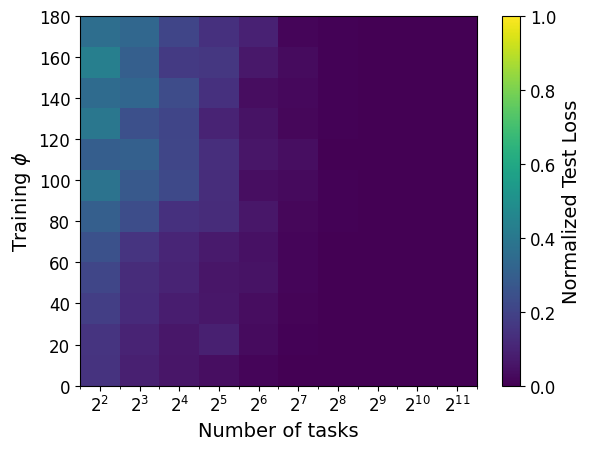} 
        \caption{} \label{fig:phase-diagram-0}
    \end{subfigure}
    \begin{subfigure}[t]{0.32\textwidth}
        \centering
        \includegraphics[width=1.0\linewidth]{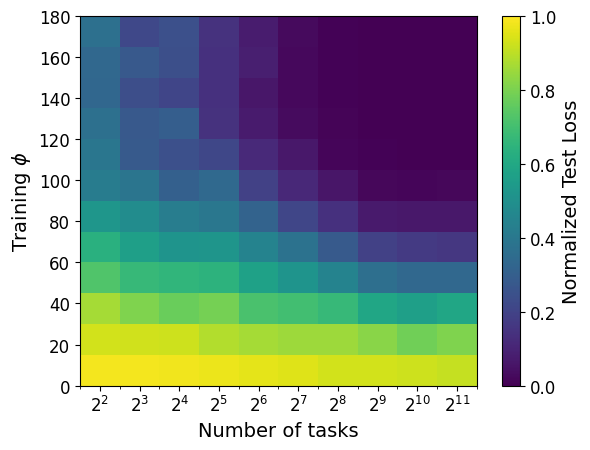} 
        \caption{} \label{fig:phase-diagram-175}
    \end{subfigure}
    \begin{subfigure}[t]{0.32\textwidth}
        \centering
        \includegraphics[width=1.0\linewidth]{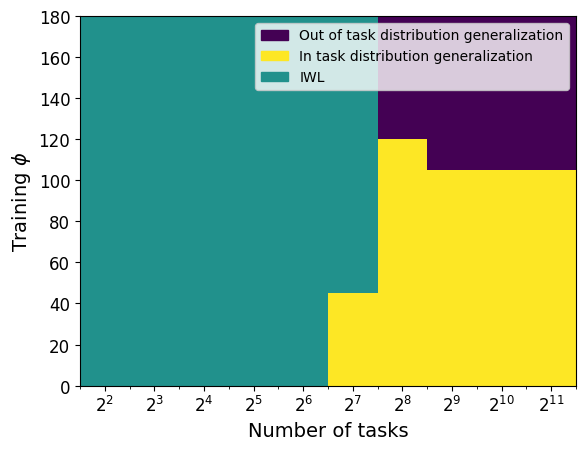} 
        \caption{} \label{fig:3phase}
    \end{subfigure}
    \caption{\textbf{Phase diagrams reveal three phases of generalization.} \textbf{A:} Phase diagram for in-task-distribution test loss ($\delta=0^\circ$). \textbf{B:} Phase diagram for out-of-task-distribution test loss ($\delta=175^\circ)$. Diagonal structure reveals interplay between two the measures of task diversity (Training $\phi$, number of tasks $N$). \textbf{C:} Combining the earlier two phase diagrams reveals three phases of in-context learning. \textit{In-weights learning, teal:} the model fits the training data but fails to generalize, both in- and out-of-task-distribution. \textit{In-task-distribution generalization, yellow:} The model generalizes within the support of the pretraining distribution, but fails to generalize out of distribution. \textit{Out-of-task-distribution generalization, purple:} The model generalizes well both in- and out-of-task-distribution. In constructing the phase diagram, we set the threshold between high and low generalization losses to $10^{-2}$.
    }
    \label{fig:phase-diagram}
\end{figure*}

\paragraph{Beyond the unit hypersphere}
What happens to the generalization ability of the model as the radius of the task distribution changes? We train several models on data generated from tasks on the surface of the unit hypersphere, and evaluate them on tasks drawn from spheres of varying radii. Each model is trained on tasks from $p_\phi(w)$ and evaluated on the equivalent distribution (with the same $\phi$) on a hypersphere with a different radius. In Fig \ref{fig:radius}, we observe that for $\phi \gtrsim 45^\circ$ the model is able to generalize perfectly to tasks with $R<1$ (shaded region), despite being trained only on tasks with $R=1$. Increasing task diversity therefore drives the model to generalize not only to new portions of the hypersphere, but also beyond the hypersphere entirely.

\subsection{Interplay between the two forms of task diversity\label{sect:phase}}
In order to examine the effect of both forms of task diversity (number of tasks and task similarity), we train 4 sets of 120 models (480 models total) with task similarity $\phi$ and number of tasks $N$ in the set: $(\phi, N) \in \{15^\circ, 30^\circ, \ldots, 180^\circ\}\times \{2^{2}, 2^{3}, \ldots, 2^{11}\}$. In Fig \ref{fig:phase-diagram-0}, we plot the resulting  \textit{in-task-distribution} loss averaged over the 4 sets: the loss for a test angle $\delta$ between $0^\circ$ and $5^\circ$ (these test angles are \textit{always} in the training task distribution). In order to more effectively compare across different $\phi$, we normalize the loss for each $\phi$ by the maximum possible loss achievable from a predictor $\hat w$ and target $w^\star$ in $S^{d-1}(\phi)$. Without such a normalization scheme, models with small $\phi$ would trivially outperform those with larger $\phi$, since even an \textit{arbitrary} choice of estimator $\hat w\in S^{d-1}$ cannot be too far from the target $w^\star$. We see that models with low $N$ and large $\phi$ perform poorly in-distribution, suggesting that the density of tasks may be important. For a more detailed analysis of the in-distribution performance of these models, see Appendix \ref{sect:dmmse}. 
% how models at the bottom left of the figure generalize well despite not being trained on many tasks
% One may wonder how models at the bottom left of the figure generalize well despite not being trained on many tasks: a partial explanation is described in section \ref{sect:spec2}.

In Fig \ref{fig:phase-diagram-175}, we plot the resulting \textit{out-of-task-distribution} loss averaged over the 4 sets, corresponding to test angles $\delta = 175^\circ$. In order to compare effectively with the normalized in-task-distribution results, we normalize these losses by dividing by $4$, the maximum loss possible when $\hat w, w^\star$ are on the sphere\footnote{To see this, set $\hat w = -w^\star$.}. We see that models with small $\phi$ perform poorly, and observe a diagonal boundary dividing models that generalize well and those that do not,
% that the boundary between the generalizing models and those that do not perform well seems to be diagonal, 
suggesting interplay between these two forms of task diversity. 
In \ref{fig:3phase}, we summarize these results as a phase diagram, depicting three distinct phases with a threshold of 0.01 as a cut-off between low loss (good generalization) and high loss (poor generalization):
% In order to combine these results into a unified phase diagram, we define three distinct phases:
\begin{enumerate}[topsep=0pt,itemsep=0ex,partopsep=1ex,parsep=1ex]
    \item Good generalization both in- and out-of-task-distribution (top right, purple).
    % If both the in task-distribution loss and out of task-distribution loss are low, the model is in a phase which generalizes well to all tasks. 
    \item Good in-task-distribution generalization, poor out-of-task-distribution generalization (bottom right, yellow).
    % If the in task-distribution loss is low but the out of task-distribution loss is high, the model is in a phase which only generalizes well in task distribution. 
    \item Poor generalization both in- and out-of-task-distribution (left, teal); the model exhibits only in-weights learning (IWL).
    % If both losses are high, the model (despite achieving low training loss) fails to generalize to new tasks, and has memorized the training tasks. We say the model exhibits only in-weights learning (IWL). 
\end{enumerate}
% In Fig \ref{fig:3phase}, we construct such a phase diagram using a threshold of $0.5$ for "low loss." We see that all three phases of learning are represented.

\paragraph{Scaling Dimension and Depth}\label{sec:scaling-dim}
In Fig \ref{fig:scaling}, we investigate how the transition changes with changes in $d$, the dimensionality of the regression problem. Unlike what is observed in \citet{raventos2023pretraining}, where as $d$ increases, a greater pretraining task diversity $N$ is necessary to induce ICL, here we observe that as $d$ increases the transition along the $\phi$ axis does not seem to change location. This suggests that the transition along this axis is not merely a result of high-dimensional geometry -- in low dimensions (e.g., $d=3$), the transition is still around $\phi\sim120^\circ$.

Similarly, we vary the number of layers in the transformer in Fig \ref{fig:depth} to study how the transition from specialization to generalization changes with respect to model depth. We find that depth does not affect the angle ($\phi\sim120^\circ$) of the transition for out-of-task-distribution test loss ($\delta = 175^\circ$). These models were trained using $N=2^{11}$ pretraining tasks. We further show in Fig \ref{fig:depth_app} in the appendix that the transition point remains the same across depth even for different testing angles $\delta$. These results further corroborate the phase diagram results in Fig \ref{fig:phase-diagram} and show that model depth does not affect the transition point. 

% \paragraph{Generalizing to new tasks in-task-distribution}
% \textcolor{red}{Call out Fig \ref{fig:interp}} and explain setup and results. Q: What does this say about the bottom left of the phase diagram? 
\subsection{Classification}
To investigate the generality of specialization-generalization transitions, we investigate the possibility of seeing a transition in a classification task: logistic regression. We now take the mapping between $x$ and $y$ to be:
\begin{align}
    y_i = H_{\frac{1}{2}}(\sigma(w^Tx_i))
\end{align}
where $H_{\frac{1}{2}}(\cdot)$ is the Heaviside step function with threshold $\frac{1}{2}$ and $\sigma(\cdot)$ is the logistic function. In Fig \ref{fig:logistic}, we observe a specialization-generalization transition for this task. The transition now occurs at $\phi\approx 135^\circ$. The fact that we observe a specialization-generalization transition in a classification setting hints that such transitions may be a more universal phenomenon.
\begin{figure}[t]
    \vskip 0.2in
    \centering
    \includegraphics{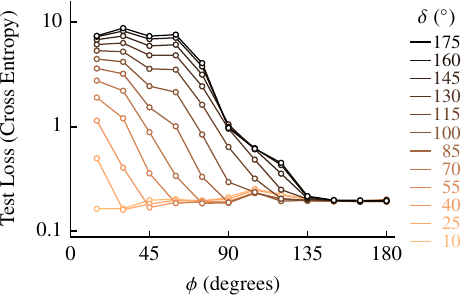}
    \caption{\textbf{Task distribution diversity induces a transition from specialized to general-purpose ICL in \textit{classification} tasks:} We now consider logistic regression. $y_i = H_{\frac{1}{2}}(\sigma(w^Tx_i))$. We see a similar transition to the one observed in the case of regression, but the transition occurs at $\phi \approx 135^\circ$.}
    \label{fig:logistic}
\end{figure}

\subsection{Nonlinear regression}
We now change the mapping between input and label for the regression to be a nonlinear function of the weights. Specifically, we consider $y_i = w_2^T\operatorname{ReLU}(W_1x_i)$, with $x_i, w_2 \in \mathbb{R}^d$ and $W_1 \in \mathbb{R}^{d\times d}$. We choose $d=3$ so that the model has $12$ parameters. In Fig \ref{fig:nonlinear}, we see that specialization-generalization transitions still occur, and investigate two ways of choosing the parameters. In Fig \ref{fig:nonlinear}A, we pick the full 12-dimensional parameter vector $\theta = \{\operatorname{vec}(W_1), w_2\}$ from the surface of $S^{11}$. This choice induces a bias towards $\norm{w_2} \ll 1$ for angles $\phi$ near the `poles' ($v=(\pm 1, \Vec{0})^T$). This bias is relaxed, however, when $\phi\sim90^\circ$, near the equator of the sphere. This leads to nonmonotonic behavior as $\delta$ changes -- the tasks near the poles are more similar to each other than to those near the equator. In Fig \ref{fig:nonlinear}A, we  only show $\delta < 90^\circ$ for this reason. In Fig \ref{fig:nonlinear_app} in the appendix, we show results for all $\delta$, and the non-monotonicity can be observed  where the red and blue lines, corresponding to $\delta$ near the poles, have lower test loss than the yellow lines, corresponding to $\delta\sim90^\circ$. 

%%%%%%%%%
\begin{figure}
    \vskip 0.2in
    \centering
    \includegraphics{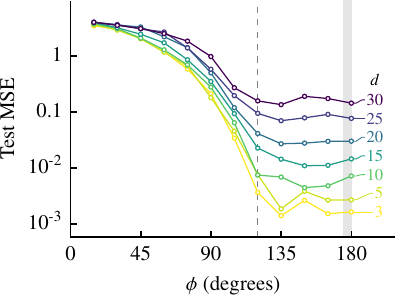}%
    \caption{\textbf{Distributional diversity threshold is unaffected by task dimension.} 
    Out-of-task-distribution test loss vs training spherical cap polar angle $\phi$ at various task dimensions $d$ (see legend). The test tasks are drawn from the spherical cap opposite to the ``pole'' of the training task distribution, $\delta = 175^{\circ}$ (shaded region).
    We see that the threshold for out-of-task-distribution generalization stays close to $\phi\sim120^\circ$ (dashed line), regardless of the task dimension. }
    \label{fig:scaling}
\end{figure}
%%%%%%%%%

%%%%%%%%%
\begin{figure}[t]
    \vskip 0.2in
    \centering
    \includegraphics[width=.8\linewidth]{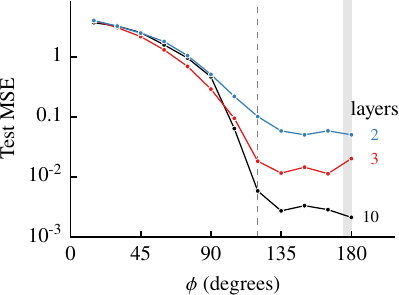}
    \caption{{\textbf{Distributional diversity threshold is unaffected
by model depth.} 
    } Out-of-task-distribution test loss vs training spherical cap polar angle $\phi$ from transformer models with two (blue), three (red), and ten (black) layers. The test tasks are drawn from the spherical cap opposite to the training task distribution, $\delta = 175^{\circ}$ (shaded region). We see that the threshold for out-of-task-distribution generalization stays close to $\phi\sim120^\circ$ (dashed line), regardless of the model depth.
    }
    \label{fig:depth}
\end{figure}
%%%%%%%%%

In contrast, in Fig \ref{fig:nonlinear}B, we pick from two separate hyperspheres: $\operatorname{vec}(W_1)\in S^8$ and $w_2 \in S^2$. This choice leads to a qualitatively similar transition to those we see in the linear case, with a transition around $\phi\sim135^\circ$. Observation of a specialization-generalization transition beyond the linear regime suggests that such transitions may be a general phenomenon in ICL. 

% \begin{figure*}[!ht]
%     \centering
%     \begin{subfigure}[t]{0.45\textwidth}
%         \centering
%         \includegraphics[width=1.0\linewidth]{task-gen-nonlinear-1.png} 
%         \caption{} \label{fig:task-gen-nonlinear-1}
%     \end{subfigure}
%     \begin{subfigure}[t]{0.45\textwidth}
%         \centering
%         \includegraphics[width=1.0\linewidth]{task-gen-nonlinear-2.png} 
%         \caption{} \label{fig:task-gen-nonlinear-2}
%     \end{subfigure}
%     \caption{\textbf{Specialization-generalization transitions in \textit{nonlinear} regression.} \textbf{A:} All parameters in the nonlinear model (a small one-hidden-layer network) are drawn from the same hypersphere. The transition occurs at $\phi\approx 45^\circ$. \textbf{B:} The parameters in the nonlinear model are drawn separately from a different hypersphere for each layer in the model. The transition occurs at $\phi\approx 90^\circ$.
%     }
%     \label{fig:task-gen-nonlinear}
% \end{figure*}

\section{Discussion and future work}
We propose another ``axis'' to training task diversity, distinct from the task diversity measure in \cite{raventos2023pretraining} (i.e., the number of pretraining tasks). This new axis of task diversity, based on the size of the subset of the task space, accounts for the similarity present between tasks. Depending on the level of task diversity present during pretraining, we have shown that transformers learn either a specialized solution that fails to generalize out-of-task-distribution, or a generic solution with good performance across the entire task space. 

While we focus on the case of learning linear tasks, the phenomena of specialization-generalization transitions are likely more general. In particular, such transitions appear also in the presence of label noise (Fig \ref{fig:transition}) and in nonlinear regression problems (Fig \ref{fig:nonlinear}). To extend our analysis to more complex tasks, it would be interesting to investigate ICL performance in richer settings as more general notions of ``task similarity'' are varied. In the context of linear problems on the sphere, the similarity between tasks $w_1$ and $w_2$ is naturally measured by their inner product $w_1^Tw_2$, but for more general problem settings it is less clear how to measure the similarity between tasks. Even within the linear setting, it may be interesting to explore task geometries beyond the sphere. 

%%%%%%%%%
\begin{figure}[t]
    \vskip 0.2in
    \centering
    \includegraphics{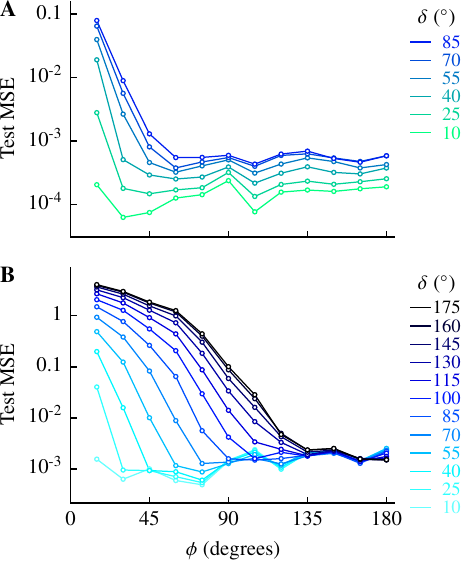}
    \caption{\textbf{Task distribution diversity induces a transition from specialized to general-purpose ICL in \textit{nonlinear} regression tasks.}
        \textbf{A:} All parameters in the nonlinear model (a small one-hidden-layer network) are drawn from the same hypersphere. The transition occurs at $\phi\approx 60^\circ$. Although the curves for each $\delta$ do not collapse onto each other after the transition, they remain within a band of size $\sim$$10^{-3}$. \textbf{B:} The parameters in the nonlinear model are drawn separately from a different hypersphere for each layer in the model. The transition occurs at $\phi\approx 135^\circ$.
    }
    \label{fig:nonlinear}
\end{figure}
%%%%%%%%%

\subsection{Types of distribution shift}
In this work, we considered how transformers' ICL performance responds to a domain shift in the task distribution. To make this precise, we can formulate ICL as a supervised learning problem. Let's just consider the linear regression setting for now. Our supervised learning dataset is then: 
\begin{equation}
    \mathcal{D} = \{C_n, y_n\}
\end{equation}
with $C_n = \{x_1,y_1,\ldots,x_n\}$, and $y_i = w^Tx_i$. The problem of task generalization we consider is then a special case of the domain generalization problem. To see this, notice that during training we see a given context $C_n$ according to a distribution $C_n \sim p(w)\prod_{i=1}^n p(x_i)$. The label function $p(y_n|C_n)$ is deterministic given the context: 
\begin{equation}
    p'(y_n|C_n) = \delta(y_n-w^Tx_n)
\end{equation}
When we ask about contexts with $w$ outside the support of the training distribution, what we have done is to change the data generating process to:
\begin{align}
    p(C_n,y_n) &= p'(C_n)p(y_n|C_n)\\
    &=  \qty[p'(w)\prod_{i=1}^n p(x_i)]\delta(y_n-w^Tx_n)
\end{align}

Notice that this is a special type of covariate shift: we do not allow the full distribution $p(C_n)$ to vary -- rather, we impose that 1) the sample generation distribution $p(x_i)$ remains unchanged; and that 2) we still draw $w$ independently from the samples $x_i$. It may be the case that this additional structure allows the network to more easily deal with the distribution shift. 

This is not the only type of distribution shift one may care about, however, and prior work has investigated ICL under different distribution shifts. \citet{ahuja2023a} investigate instead a domain shift of the $x_i$'s, (i.e. $p(x_i)\rightarrow p'(x_i)$ and find that transformers fail to generalize when exposed to such shifts. We suggest that this may be because the authors do not expose the model to sufficient \textit{data} diversity during pretraining. In Appendix \ref{sect:xs-cone} we investigate the effect of data diversity in the linear regression setting and see a (noisy) transition in the behavior of the model's OOD generalization. \citet{yadlowsky2024can, yadlowsky2023pretrainingdatamixturesenable} instead consider ICL under \textit{concept shift}: the mapping of data to labels $p(y_n|C_n)$ changes between train and test. The authors here find that transformers largely fail to generalize to such distribution shifts, even when exposed to a diversity of tasks during pretraining. \citet{hill2025transformersdontincontextlearn} consider ICL of linear functions where task vectors at test time are orthogonal to those seen during training, and again see that transformers fail to generalize to this distribution shift. We expect that this is because the natural task similarity measure for linear tasks, the inner product, is zero for orthogonal tasks. The ability our models display in successfully adapting to distribution shift between train and test time may therefore depend on the type of distribution shift considered, in addition to task/data diversity thresholds.

\subsection{Future directions}

A natural next step would be to develop an analytic theory for analyzing specialization-generalization transitions in transformers, similar to the analysis \citet{lu2024asymptotictheoryincontextlearning} performed for the setting studied in \citet{raventos2023pretraining}. Such a model could provide further insight into the conditions   necessary for out-of-task-distribution ICL to emerge. Additionally, it would be interesting to study the dynamics of how out-of-task-distribution ICL emerges during training, shedding light on the learning processes \& implicit biases of the optimization algorithms that enable models to generalize beyond their training task distribution.

Our experiments open a new direction for understanding how general-purpose models are able to solve unseen tasks using only a few examples in their context: we show empirically that transformers can learn to do ICL over much more of the task space than they are trained on. Understanding the generality of this behavior may help explain why language models are able to perform well on ICL tasks not present in their pretraining distribution. Although our experiments here are limited by their focus on relatively simple functions as the ICL task, we believe investigations into specialization-generalization transitions for more complex tasks are a promising direction for future study. Building trust in LLMs is an important challenge with positive societal impacts, and understanding the degree and nature of task generalization via ICL takes a step towards this goal.

\section*{Impact statement}
We present work on linear regression ICL tasks with the goal of understanding more broadly the factors that lead to the development of ICL behavior. While we focus on simplified experimental settings compared to the more complex ICL phenomena seen in LLMs, the impact of understanding the underpinnings of specialization and generalization in ICL will provide insights into the learning dynamics of LLMs, building trust in AI and producing positive consequences for society. 

\section*{Acknowledgments}
CG acknowledges travel support from the National Science Foundation and by DoD OUSD (R\&E) under Co-operative Agreement PHY-2229929 (The NSF AI Institute for Artificial and Natural Intelligence). LMS is supported by the National Science Foundation Graduate Research Fellowship Program under Grant No. DGE-2039656. VN acknowledges research funds from the University of Sydney. 
DJS was partially supported by a Simons Fellowship in the MMLS, a Sloan Fellowship, and the National Science
Foundation, through the Center for the Physics of Biological Function (PHY-1734030).

\bibliography{icml2025}
\bibliographystyle{icml2025}
\clearpage
\appendix
\section{Appendix / supplemental material}
\subsection{Further experimental details \label{sect:more-expt}} 
All code was written in Python using the PyTorch library \cite{pytorch}. All models were trained for 58,000 steps using a batch size of 128 and a constant learning rate of $3\times 10^{-4}$. All models were converged at the end of training. All models were trained on a single GPU, either a MIG GPU with 10GB of memory or an A100 with 40GB of memory, and took $\sim3$hrs to train.

\begin{figure}[!h]
    \centering
    \includegraphics[width=0.9\linewidth]{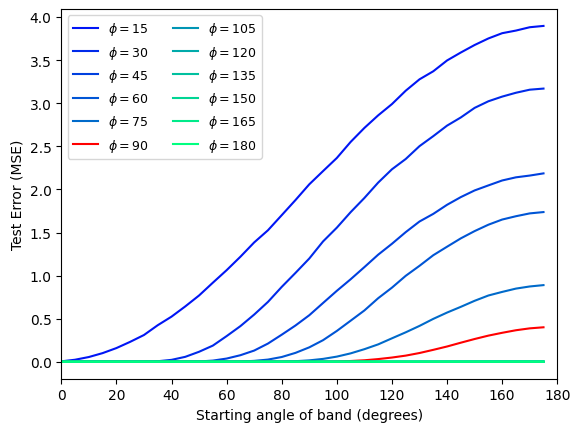}
    \caption{\textbf{A second run of Fig \ref{fig:transition}, with a different initialization and sampling of $w\sim p_\phi$.}}
    \label{fig:run2}
\end{figure}

\begin{figure}[!ht]
    \centering
    %\begin{subfigure}[t]{\linewidth}
    %    \centering
        \includegraphics[width=0.9\linewidth]{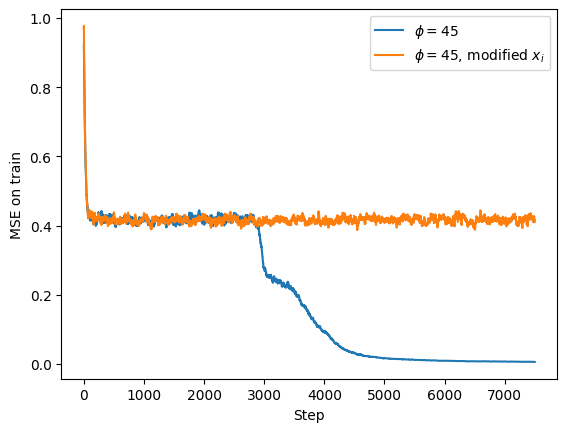}
    %    \caption{}
    %    \label{fig:spec2}
    %\end{subfigure}
    %\hfill
    %\\[\baselineskip]
    % \begin{subfigure}[t]{\linewidth}
    %     \centering
    %     \includegraphics[width=1.0\linewidth]{spec2-lowtask.png}
    %     \caption{}
    %     \label{fig:spec2-lowtask}
    % \end{subfigure}
    \caption{\textbf{Transformers undergo two stages of specialization during training:} \textbf{A:} For small $\phi$, the transformer rapidly (within the first epoch) learns a solution that only takes into account the component of $x_i$ in the direction of the vector $v$ forming the center of the hyperspherical cap. \textit{Blue:} A transformer trained normally on training data with $\phi = 45^\circ$. \textit{Orange:} A transformer trained on data with the components of $x_i$ perpendicular to $v$ zeroed out. The training loss is smoothed with an exponential moving average for clarity of visualization.
    % \textbf{B:} For low task number, the in-task distribution loss (\textit{orange}) tracks the loss for a regression weight vector $w = v$ (\textit{blue}).
    }
    \label{fig:spec2}
\end{figure}

\subsection{Two stages of specialization during training \label{sect:spec2}}

In Fig \ref{fig:spec2}, we compare the training loss of a transformer trained normally on data with $\phi = 45^\circ$ to a transformer trained on modified data. To modify the data, we zero out all components of $x_i$ that are perpendicular to the vector $v$ defining the center of the hyperspherical training cap. We see that during early stages of training, the transformer trained on unmodified data performs similarly to the transformer trained on modified data, suggesting that early in training, transformers trained to do linear regression only take into account the component of $x_i$ parallel to $v$. Later in training, the unmodified transformer learns to take into account other directions in the training data. This suggests that there are two distinct specialized solutions learned by transformers when $\phi$ is small, the first of which is transient and disappears after training long enough.

\subsection{Defining the transition point}
To more precisely define the transition point, we quantify the degree to which a model trained with a given $\phi$ performs similarly across test angles $\delta$. The intuition for this definition is that for a model with $\phi$ above the transition point, we should expect similar performance across all $\delta$ due to the rotational symmetry in the linear regression problem. To quantify this, we measure the standard deviation of the model's performance across $\delta$, and normalize this by the mean performance across $\delta:$
\begin{align}
    \text{NSR} = \frac{\sqrt{\operatorname{Var}_\delta\qty[\mathbb{E}_{w\sim p_{\delta, \Delta \delta}}[(T_\theta(C_n)-y_n)^2]]}}{\mathbb{E}_\delta\qty[\mathbb{E}_{w\sim p_{\delta, \Delta \delta}}[(T_\theta(C_n)-y_n)^2]]}\label{eqn:NSR}
\end{align}
This definition resembles an inverse signal-to-noise ratio, or ``NSR''. To identify the transition point, we then look for a sharp drop in the NSR. In Fig \ref{fig:NSR-1}, we plot the NSR against $\phi$ for the models in Fig \ref{fig:transition}A. We see a sharp drop in the NSR on a logarithmic scale, and identify which phase of learning we are in by thresholding the NSR. The fact that the drop is sharp means that our phase identification is not very sensitive to our choice of threshold. 
\begin{figure}
    \centering
    \includegraphics[width=0.9\linewidth]{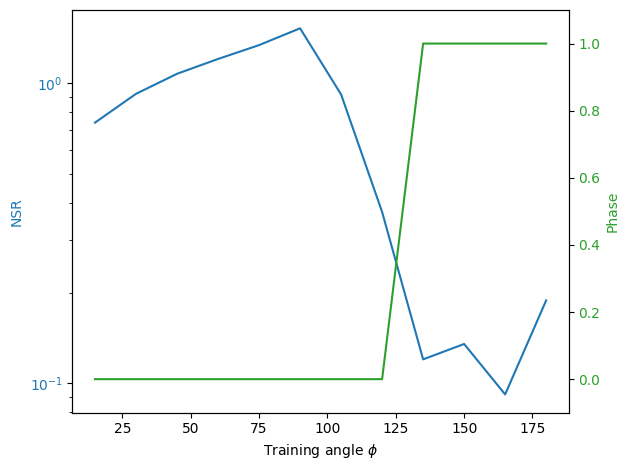}
    \caption{\textit{Blue:} NSR (Eqn \ref{eqn:NSR}) vs $\phi$ for the models in Fig \ref{fig:transition}A. \textit{Green:} Phase identification by thresholding the NSR, with a threshold of 0.5.}
    \label{fig:NSR-1}
\end{figure}

\subsection{Comparison with dMMSE estimator \label{sect:dmmse}}
Plugging the uniform distribution over a finite pretraining set $\mathcal{W} = \{w_1,w_2\ldots,w_N\}$ into Eqn \ref{eqn:opt-w} yields the discrete minimum mean-squared error (dMMSE) estimator (for the case of zero label noise):
\begin{equation}
    \hat w_\text{dMMSE} = \operatorname{arg}\min_{w\in \mathcal{W}}\sum_{i=1}^n\qty(w^Tx_i-y_i)^2  \label{eqn:dMMSE}
\end{equation}
In Fig \ref{fig:interp}A, we investigate the in-distribution performance of pretrained transformers by comparing their performance to a dMMSE estimator. To do this, we interpolate between tasks in $\mathcal{W}$ along a great circle, and test both the transformer and dMMSE along this path. For a low number of tasks, transformers perform about as well as dMMSE, but transformers outperform dMMSE for larger numbers of tasks. 

In Fig \ref{fig:interp}B, we plot a heatmap of the quantity:
\begin{equation}
    D = \max_\alpha\qty[\qty(\hat w_\text{dMMSE}^Tx_\alpha - y_\alpha)^2-\qty(T_\theta(x_\alpha) - y_\alpha)^2]\label{eqn:excess}
\end{equation}
where $x_\alpha, y_\alpha$ is the regression problem generated by the interpolant weight vector $w_\alpha$ at interpolation step $\alpha$. The quantity $D$ measures the performance of the transformer relative to the dMMSE estimator: when $D<0$, the transformer outperforms dMMSE. In the heatmap, we see that increasing number of tasks drives the transformer to a solution that outperforms dMMSE for all $\phi$. In order to compare across different values of $\phi$, we normalize $D$ using the scheme described in section \ref{sect:phase}. Notice that this normalization aligns the transition points for the transition from IWL to ICL across increasing number of tasks. 

\begin{figure}[!ht]
    \centering
    \begin{subfigure}[t]{1.0\linewidth}
        \centering
        \includegraphics[width=0.9\linewidth]{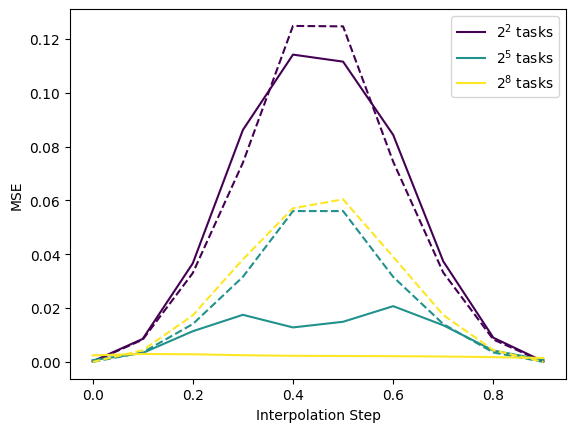} 
        \caption{} %\label{fig:interp}
    \end{subfigure}
    \begin{subfigure}[t]{1.0\linewidth}
        \centering
        \includegraphics[width=0.9\linewidth]{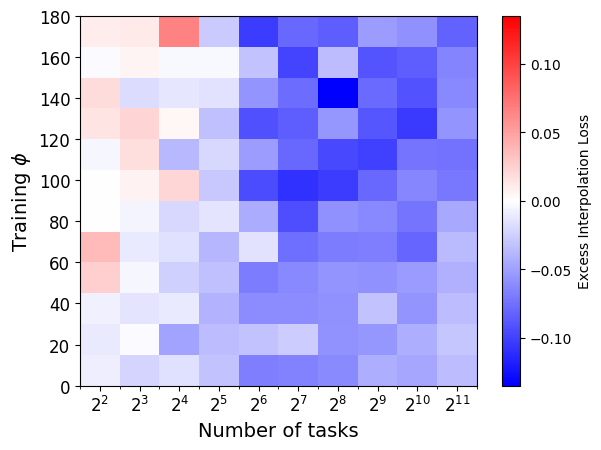} 
        \caption{} \label{fig:excess}
    \end{subfigure}
    \caption{\textbf{In-task-distribution generalization.} \textbf{A:} \textit{Solid:} The loss for a transformer with $\phi = 30^\circ$ along a great circle connecting two weight vectors in the pretraining set. \textit{Dashed:} The loss along the same great circle for the dMMSE estimator (Eqn \ref{eqn:dMMSE}) corresponding to the pretraining set. \textbf{B:} Transformers outperform dMMSE for all $\phi$ with sufficiently many pretraining tasks. The color shows the excess loss of the transformer over the dMMSE estimator. When it is negative, the transformer outperforms dMMSE (see Eqn \ref{eqn:excess}).}
    \label{fig:interp}
\end{figure}

\subsection{Sampling from a portion of the hypersphere\label{sect:sampling}}
The curse of dimensionality precludes sampling from a portion of the hypersphere via rejection sampling. Instead, we consider the problem of sampling uniformly from the intersection of the sphere with a cone in $d+1$ dimensions: i.e. sampling from the sphere $S^d(R)$ subject to the restriction:
\begin{align}
    \norm{w - v}^2 \geq r^2 \Leftrightarrow \operatorname{angle}(w, v) \leq \theta \Leftrightarrow w^Tv \geq R^2 - r^2/2
\end{align}
where $w$ is our sampled vector and $v$ is a fixed vector that defines the hypercone. The opening half-angle of the cone $\theta$ is related to $r$ via $r^2 = 2R^2\qty(1-\cos\theta)$.

WLOG, we take $v = R\hat e_1 = (R \ \vec 0)$, where $\vec 0 \in \mathbb{R}^d$. Then, writing $w = (Z \ y)$ with $Z \in \mathbb{R}$ and $y \in \mathbb{R}^d$, we see that 
\begin{align}
    RZ \geq R^2 - r^2/2
\end{align}
We will show that, up to this restriction, $Z = R(2U-1)$, with $U\sim \operatorname{Beta}(\frac{d}{2},\frac{d}{2})$. To generate $y$, we first draw $Z$, then draw a vector uniformly from the sphere $S^{d-1}(\sqrt{R^2 - Z^2})$
\begin{proposition}
    $Z = R(2U-1)$, with $U\sim \operatorname{Beta}(\frac{d}{2},\frac{d}{2})$
\end{proposition}
\begin{proof}
    Observe that $p(Z)\dd{Z}$ is proportional to the ($d$-dimensional) surface area of the conical frustum with height $\dd{Z}$ and radii $\sqrt{R^2 - Z^2}$, $\sqrt{R^2 - (Z+\dd{Z})^2}$. Expanding to first order in $\dd{Z}$, the slant height $\ell$ of the frustum is given by:
    \begin{align}
        \ell^2 = (\dd{Z})^2\qty[\frac{R^2}{R^2-Z^2}]
    \end{align}
    so that:
    \begin{align}
        p(Z)\dd{Z} &\propto \ell \qty(\sqrt{R^2-Z^2})^{d-1} \\
        &\propto (R^2-Z^2)^{(d-2)/2}\dd{Z}
    \end{align}
    Making the change of variables $Z = R(2u-1)$, we obtain:
    \begin{align}
        p(u)\dd{u} &\propto (u-u^2)^{(d-2)/2}\dd{u} \\
        &\propto u^{\frac{d}{2}-1}(1-u)^{\frac{d}{2}-1}\dd{u}
    \end{align}
    which has the form of a $\operatorname{Beta}(\frac{d}{2},\frac{d}{2})$ random variable. 
\end{proof}
\subsubsection{Sampling $Z$}
As we have seen, $Z$ follows a scaled \& shifted $\operatorname{Beta}(\frac{d}{2},\frac{d}{2})$ distribution, except for the constraint that $RZ\geq R^2-r^2/2$. To implement this constraint, observe first that it is equivalent to:
\begin{align}
    U \geq 1-\frac{r^2}{4R^2}
\end{align}
Since the distribution of $U\sim \operatorname{Beta}(\frac{d}{2},\frac{d}{2})$ is symmetric about 1/2, this condition is equivalent to:
\begin{align}
    U \leq \frac{r^2}{4R^2}
\end{align}
It follows that to sample from $Z$, we can perform the following sequence of transformations:
\begin{align}
    T &\sim \operatorname{Unif}\qty(0,F\qty(\frac{r^2}{4R^2})) \\
    U &= F^{-1}(T) \\
    Z &= R(2U-1)
\end{align}
where $F(\cdot)$ is the cdf of the $\operatorname{Beta}(\frac{d}{2},\frac{d}{2})$ distribution. 
\subsubsection{Minimum angles\label{sect:sample-min}}
It is straightforward to implement an additional constraint corresponding to a minimum distance/angle away from the vector $v$. If $r'$ is this minimum distance, then we perform: 
\begin{align}
    T &\sim \operatorname{Unif}\qty(F\qty(\frac{(r')^2}{4R^2}),F\qty(\frac{r^2}{4R^2})) \\
    U &= F^{-1}(T) \\
    Z &= R(2U-1)
\end{align}

% In Fig \ref{fig:spec2-lowtask}, we compare the in-task-distribution test loss for a model with the regression weight vector $w = v$ with the in-task distribution loss for models trained on a low number of tasks. We see that these two quantities closely track each other, suggesting that transformers are able to learn the first stage of specialized solution even when the number of tasks is low. 

% \subsection{Gaussian $w$'s}
% \begin{figure}
%     \centering
%     \includegraphics[width=0.5\linewidth]{gaussian_data}
%     \caption{Caption}
%     \label{fig:gaussian_data}
% \end{figure}

\subsection{Data diversity thresholds \label{sect:xs-cone}}
In this section we investigate the level of diversity necessary in the \textit{data}: i.e. when we instead draw $x_i \sim S^{d-1}(\phi)$, with $w$ drawn \textit{uniformly} over the \textit{entire} unit sphere. In Fig \ref{fig:xs-cone}, we show the results. We still observe a transition, suggesting that sufficient diversity is also required in the data in order to generalize out-of-distribution. We notice that the transition is much noisier than that observed in Fig \ref{fig:transition}. This difference in behavior may be due to the fact that the model directly observes $x_i$ in the context, whereas the task vector $w$ is a latent variable. 

\begin{figure}
    \centering
    \includegraphics[width=0.9\linewidth]{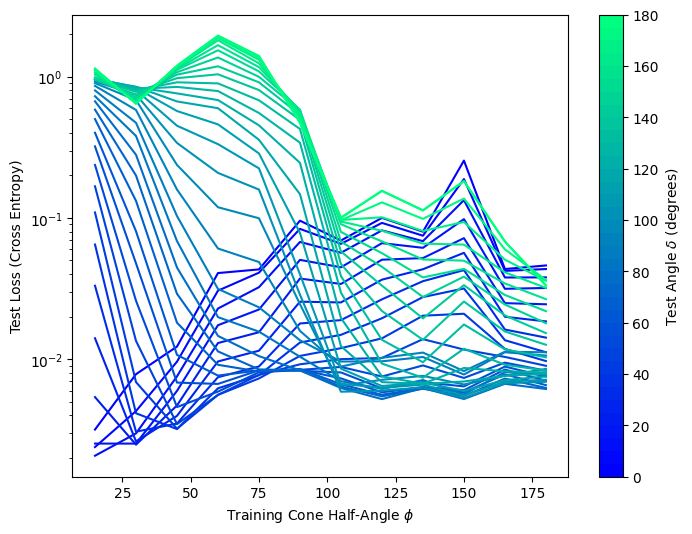}
    \caption{\textbf{Data diversity transition:} When instead drawing $x_i$ from a hyperspherical cap, we see a (noisy) specialization-generalization transition at $\phi \approx 105^\circ$.}
    \label{fig:xs-cone}
\end{figure}

\begin{figure}
    \centering
    \includegraphics[width=0.9\linewidth]{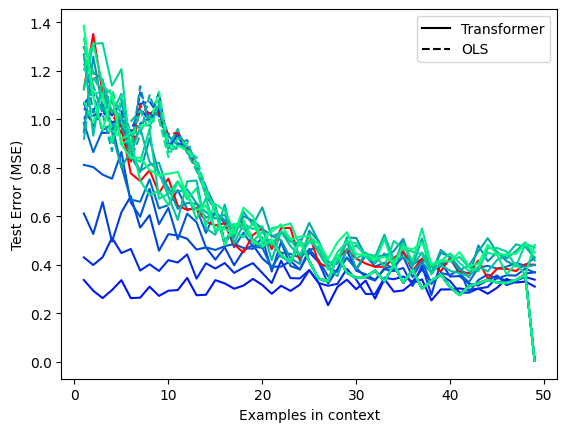}
    \caption{\textbf{In-distribution test error vs context length for the noisy case ($\sigma^2=\frac{1}{4}$).}}
    \label{fig:noisy-context}
\end{figure}
\begin{figure}
    \centering
    \includegraphics[width=0.9\linewidth]{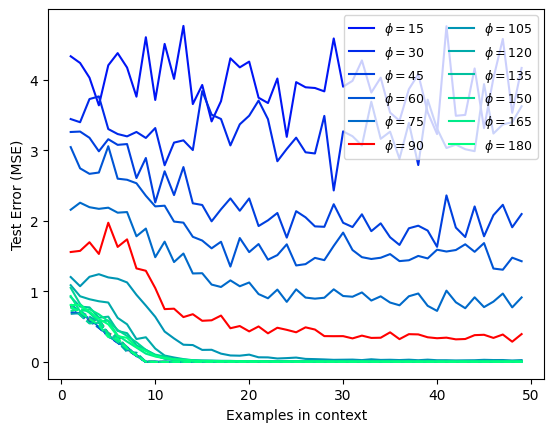}
    \caption{\textbf{Out-of-distribution ($\delta=175^\circ$) test error vs context length.}}
    \label{fig:ood-context}
\end{figure}
%%%%%%%%%

\subsection{Different Context Lengths}
In Figures \ref{fig:noisy-context} and \ref{fig:ood-context}, we investigate models trained with various $\phi$ across many context lengths in different settings to the setting in Fig \ref{fig:contextlength}.

\subsection{Supplemental plots for section \ref{sect:phase}}
In Fig \ref{fig:depth_app}, we show the data from Fig \ref{fig:depth}, but for all $\delta$, highlighting that the transition point does not change with increasing model depth. 

In Fig \ref{fig:nonlinear_app}, we reproduce Fig \ref{fig:nonlinear}, but show many values of $\delta$, highlighting the non-monotonic behavior that arises as a result of the rotational symmetry of the sphere being violated by the nonlinearity.

%%%%%%%%%

%%%%%%%%%
\begin{figure}[b]
    \centering
    \includegraphics[width=0.8\linewidth]{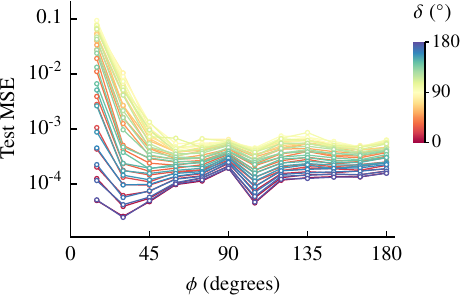}
    \caption{
        {Same as Fig~\ref{fig:nonlinear}A, but shows all values of $\delta$. Notice how the behavior of the test loss is non-monotonic as $\delta$ increases.}
    }
    \label{fig:nonlinear_app}
\end{figure}
%%%%%%%%%

\begin{figure}
    \centering
    \includegraphics[width=0.95\linewidth]{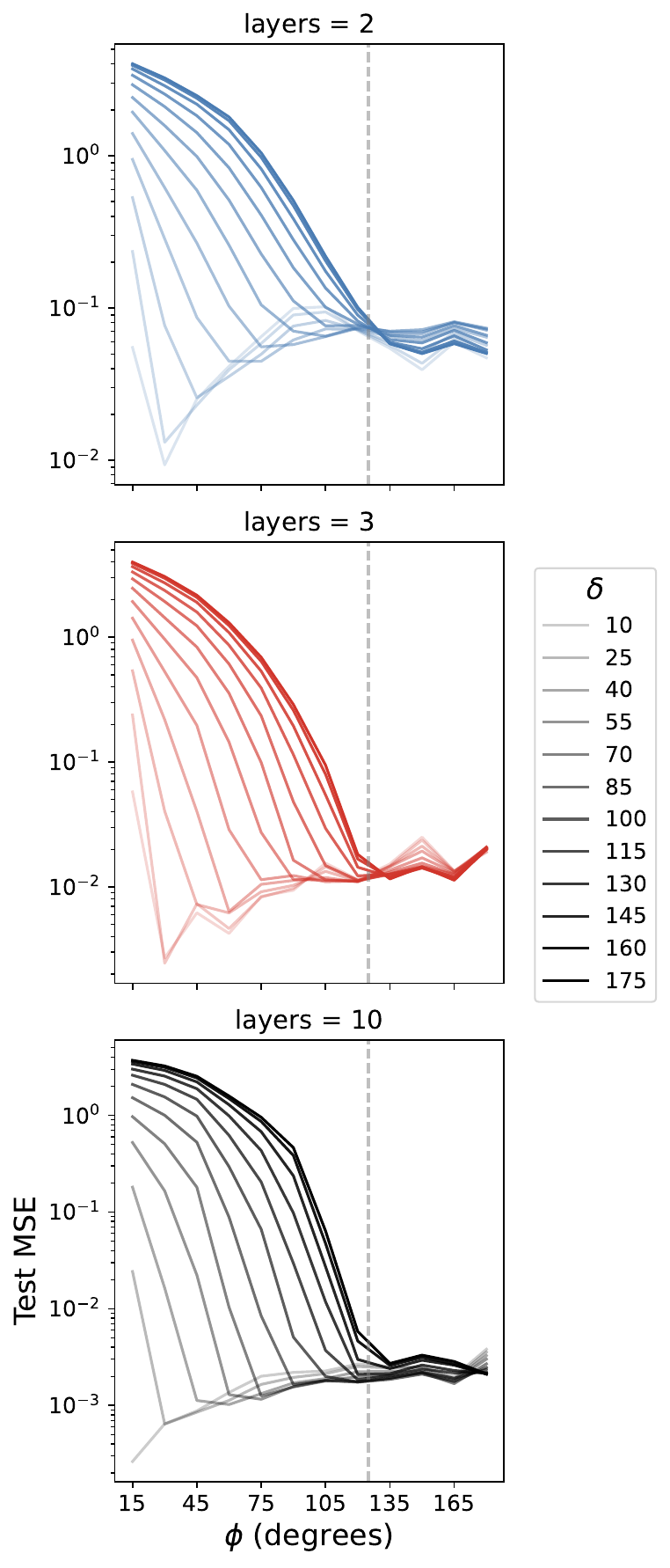}
    \caption{{\textbf{Distributional diversity threshold is unaffected
by model depth, even across test angles.} 
    } Test loss vs training spherical cap polar angle $\phi$ from transformer models with two (blue), three (red), and ten (black) layers. We show results for test tasks drawn from $\Delta\delta = 5^\circ$ bands from $\delta = 15^\circ$ to $175^\circ$, where lighter shades indicate smaller $\delta$. We see that the threshold for out-of-task-distribution generalization stays close to $\phi\sim120^\circ$, regardless of the model depth.
    }
    \label{fig:depth_app}
\end{figure}

\end{document}